\documentclass{article}
\usepackage{iclr2027_conference,times}

\usepackage{amsmath,amsfonts,bm}

\def\eqref#1{equation~\ref{#1}}
\def\1{\bm{1}}

\DeclareMathAlphabet{\mathsfit}{\encodingdefault}{\sfdefault}{m}{sl}
\SetMathAlphabet{\mathsfit}{bold}{\encodingdefault}{\sfdefault}{bx}{n}

\usepackage[hypertexnames=false]{hyperref}
\usepackage{url}
\usepackage{booktabs}
\usepackage{enumitem}
\usepackage{amsmath,amssymb,amsthm}
\usepackage{algorithm,algpseudocode}
\usepackage{mathtools}
\usepackage{makecell}
\usepackage{graphicx}
\usepackage{wrapfig}
\usepackage{subcaption}
\usepackage{fontawesome5}
\usepackage[table]{xcolor}
\usepackage{tabularx}
\usepackage{tikz}

\definecolor{RAOne}{HTML}{6FA99A}
\definecolor{RATwo}{HTML}{72A8AD}
\definecolor{RAThree}{HTML}{7D9FBB}
\definecolor{RAFour}{HTML}{8D97C1}
\definecolor{RAFive}{HTML}{A391BC}
\definecolor{RASix}{HTML}{B38FAA}

\tikzset{
  ra/category/.style={
    rounded corners=1.5pt,
    line width=0.35pt,
    fill opacity=0.08,
    draw opacity=1,
    text opacity=1,
    text=black!85,
    inner xsep=3.5pt,
    inner ysep=2pt,
    outer sep=0pt,
    text height=1.7ex,
    text depth=0.45ex,
    font=\sffamily\scriptsize
  }
}

\newcommand{\RAGroup}[3]{%
  \textcolor{#1!55!black}{\textbf{#2: #3}}%
}

\newcommand{\RACategory}[3]{%
  \tikz[baseline=(ra-tag.base)]{%
    \node[ra/category, draw=#1!75!white, fill=#1] (ra-tag)
      {\textcolor{#1!55!black}{\textbf{#2}}\enspace #3};%
  }%
}

\newtheorem{definition}{Definition}
\newtheorem{theorem}{Theorem}

\newtheorem{proposition}[theorem]{Proposition}

\definecolor{tabband}{RGB}{240,240,245}

\title{ARIA: Runtime Inspection and Amendment for Multi-Agent Systems}

\author{
Bolin Shen\\
Florida State University
\And
Yushun Dong\thanks{Corresponding author.}\\
Florida State University
}

\iclrfinalcopy

\begin{document}

\maketitle

% \ariabadge
%   {https://anonymous.4open.science/w/aria-project}
%   {https://github.com/LabRAI/ARIA}
%   {https://huggingface.co/datasets/anonymous/aria}

\begin{abstract}
In LLM-based multi-agent systems (MAS), a local mistake can propagate as other agents build on it, turning a correctable error into a failed task. Diagnosing completed runs or repeatedly re-executing them to search for corrections does not provide a learned policy for intervening while execution is still in progress. We present Agent Runtime Inspection and Amendment (ARIA), a framework for learning when and how to repair ongoing multi-agent executions. Our key intuition is to turn execution-verified repairs into reusable supervision, moving repair search into data construction and training so that a learned model can intervene at runtime without repeatedly testing alternatives. First, we collect over 24{,}000 successful and failed trajectories across four multi-agent frameworks, six benchmarks, and eight host models. From these trajectories, we construct the ARIA dataset, to our knowledge the first large-scale dataset pairing failed multi-agent trajectories with execution-verified repairs. It contains 3{,}953 verified repair pairs, each specifying what went wrong, where to intervene, and how execution should continue. Then, we develop the ARIA framework for MAS-agnostic runtime inspection and amendment, designed to integrate with any MAS that exposes execution trajectories and supports continuation replacement. Using verified repairs as training examples, we teach a compact model to recognize when intervention is needed and how to correct an ongoing run. The aim is to learn repair strategies that apply to related failures in new configurations, rather than merely reproduce individual corrections. Finally, the trained ARIA model demonstrates strong repair performance in both post-hoc and runtime settings, with generalization to unseen frameworks and host models. Our code is available at \url{https://github.com/LabRAI/ARIA}.
\end{abstract}

% =========================================================================
% Introduction: Multi-Agent Systems and the Need for Runtime Repair
% =========================================================================

\section{Introduction}
\label{sec:introduction}

The use of large language models (LLMs) has expanded from text generation and dialogue~\citep{brown2020gpt3,touvron2023llama2} to agents that plan, invoke tools, and adapt to feedback~\citep{yao2023react,schick2023toolformer,shinn2023reflexion}. Multi-agent systems (MAS) organize agents into coordinated workflows with distinct roles, decomposed tasks, and exchanges through messages or shared artifacts~\citep{li2023camel,wu2023autogenenablingnextgenllm,chen2024agentverse,guo2024massurvey}. This division of labor combines specialized capabilities with cross-agent critique to refine intermediate reasoning~\citep{du2024debate,liang2024debate}. Applications include software development, where MetaGPT and ChatDev coordinate design, implementation, and testing~\citep{hong2024metagpt,qian2024chatdev}; scientific experimentation, where Coscientist integrates planning, search, code execution, and laboratory tools~\citep{boiko2023coscientist}; medical reasoning with MedAgents and MDAgents~\citep{tang2024medagents,kim2024mdagents}; interactive social simulation~\citep{park2023generativeagents}; and general-purpose web and file tasks with Magentic-One~\citep{fourney2024magenticone}. However, distributed work creates dependencies through which an incorrect intermediate result may become the premise for another agent's reasoning or the input to a later tool call~\citep{hong2024metagpt,guo2024massurvey}. Misunderstood messages, inconsistent assumptions, or inadequate verification can carry errors across handoffs, turning a local mistake into an execution-level failure~\citep{cemri2025mast,hammond2025multiagentrisks}. These dependencies motivate runtime inspection and repair before actions compound a faulty state, rather than relying solely on final-output evaluation.

Work on execution reliability addresses three questions: what fails, where it originates, and how successful execution can be restored. Failure taxonomies characterize recurring patterns: the Multi-Agent System Failure Taxonomy (MAST)~\citep{cemri2025mast} identifies 14 modes spanning system design, inter-agent misalignment, and task verification, while a complementary analysis~\citep{hammond2025multiagentrisks} examines interaction-level risks from miscoordination, conflict, and collusion. Incident-level diagnosis instead locates the decisions responsible for failed runs. The Who\&When benchmark~\citep{zhang2025whichagent} evaluates agent- and step-level failure attribution; AgentRx~\citep{barke2026agentrx} uses constraints derived from task requirements and tool specifications to identify critical failure steps; and ErrorProbe~\citep{li2026errorprobe} combines backward tracing with execution-based verification and retains verified error patterns. Repair methods then test whether changing suspected failure points improves outcomes. CausalFlow~\citep{bonagiri2026causalflow} searches for minimal counterfactual edits to single-agent traces, whereas DoVer~\citep{ma2026dover} tests message or plan changes in multi-agent systems by re-executing the affected continuation. These works support retrospective diagnosis and repair, but runtime intervention must determine from an incomplete trajectory whether a failure has occurred and how to amend the continuation. Learning this capability from past executions presents three connected challenges. First, reliable supervision must pair a failed trajectory with a well-supported diagnosis and an effective correction: task-level labels do not identify the responsible decisions~\citep{zhang2025whichagent,barke2026agentrx}, and plausible repairs require evidence of restored success. Second, collecting such supervision at scale requires selecting productive sources across MAS frameworks, tasks, and backbone models. Configurations that rarely fail supply few examples, yet frequent failures may still yield few verified repairs when correction is difficult; task accuracy alone therefore cannot identify sources rich in useful failure--repair pairs. Third, incident-specific corrections do not directly specify how to act on unseen, unfinished trajectories. A runtime model must decide whether a prefix warrants intervention and generate an amendment that improves subsequent execution. Together, these challenges call for connecting reliable supervision, repair-aware data collection, and joint learning of runtime inspection and amendment.

To address these challenges, we propose Agent Runtime Inspection and Amendment (ARIA), which couples an execution-verified dataset with a framework for learning when and how to intervene in multi-agent executions. We collect over 24{,}000 successful and failed trajectories across four MAS frameworks, six benchmarks, and eight host models. To obtain useful supervision efficiently, we introduce a source selection rule based on repair optimality that balances failure frequency with repairability, favoring configurations that yield more verified repairs per execution. We connect failure diagnoses to targeted corrections and retain a pair only when the repaired execution passes benchmark evaluation while its unmodified counterpart fails. This process produces the ARIA dataset of 3{,}953 verified failure--repair pairs, each recording what went wrong, where to intervene, and how execution should continue. Building on this dataset, the ARIA framework trains a compact model in three stages. First, supervised learning teaches the model to recognize when intervention is needed and to generate repairs from verified examples. Second, preference learning~\citep{rafailov2023dpo} favors successful repairs over failed continuations and diagnoses that help predict repairs for related incidents. Third, policy optimization uses execution feedback on the model's own proposals to improve repair effectiveness and cross-configuration transfer, while refining intervention decisions to recover failures and limit harmful interventions. The resulting model supports both runtime inspection and amendment and post-hoc repair.

\paragraph{Contributions.}
Our contributions are threefold:
\begin{itemize}[leftmargin=2em]
    \item \textbf{A large-scale failure repair dataset.} To the best of our knowledge, ARIA is the first large-scale dataset pairing failed multi-agent trajectories with execution-verified repairs. It spans four MAS frameworks, six benchmarks, and eight backbone models, enabling systematic study of failure repair across diverse multi-agent systems.

    \item \textbf{A novel MAS-agnostic runtime repair framework.} We introduce ARIA, to the best of our knowledge the first MAS-agnostic framework for runtime failure repair. ARIA operates on partial execution trajectories to detect and repair failures without access to the host system's internals, enabling integration into diverse MAS frameworks.

    \item \textbf{State-of-the-art repair performance.} Our post-training approach equips a compact model with both a gate for runtime error detection and a patching capability for failure repair. Extensive experiments show that even a 7B ARIA model achieves state-of-the-art executed repair accuracy in both runtime and post-hoc settings, outperforming zero-shot frontier models.
\end{itemize}

% =========================================================================
% Methodology: Dataset Creation and Post-Training
% =========================================================================

\section{Methodology}
\label{sec:methodology}

ARIA consists of a corpus and a model trained on it. Section~\ref{subsec:dataset-creation} determines which source configurations to execute and records all resulting source trajectories, and Section~\ref{subsec:post-training} describes the three-stage offline procedure that turns those records into a predictor.

% =========================================================================
% Dataset Creation subsection for inclusion in the main manuscript.
% Required table packages in the main preamble:
% \usepackage{booktabs,tabularx}
% The main manuscript supplies the existing theorem environments and bibliography.

\subsection{Dataset Creation}
\label{subsec:dataset-creation}

\begin{wrapfigure}{r}{0.55\textwidth}
\centering
\includegraphics[width=\linewidth]{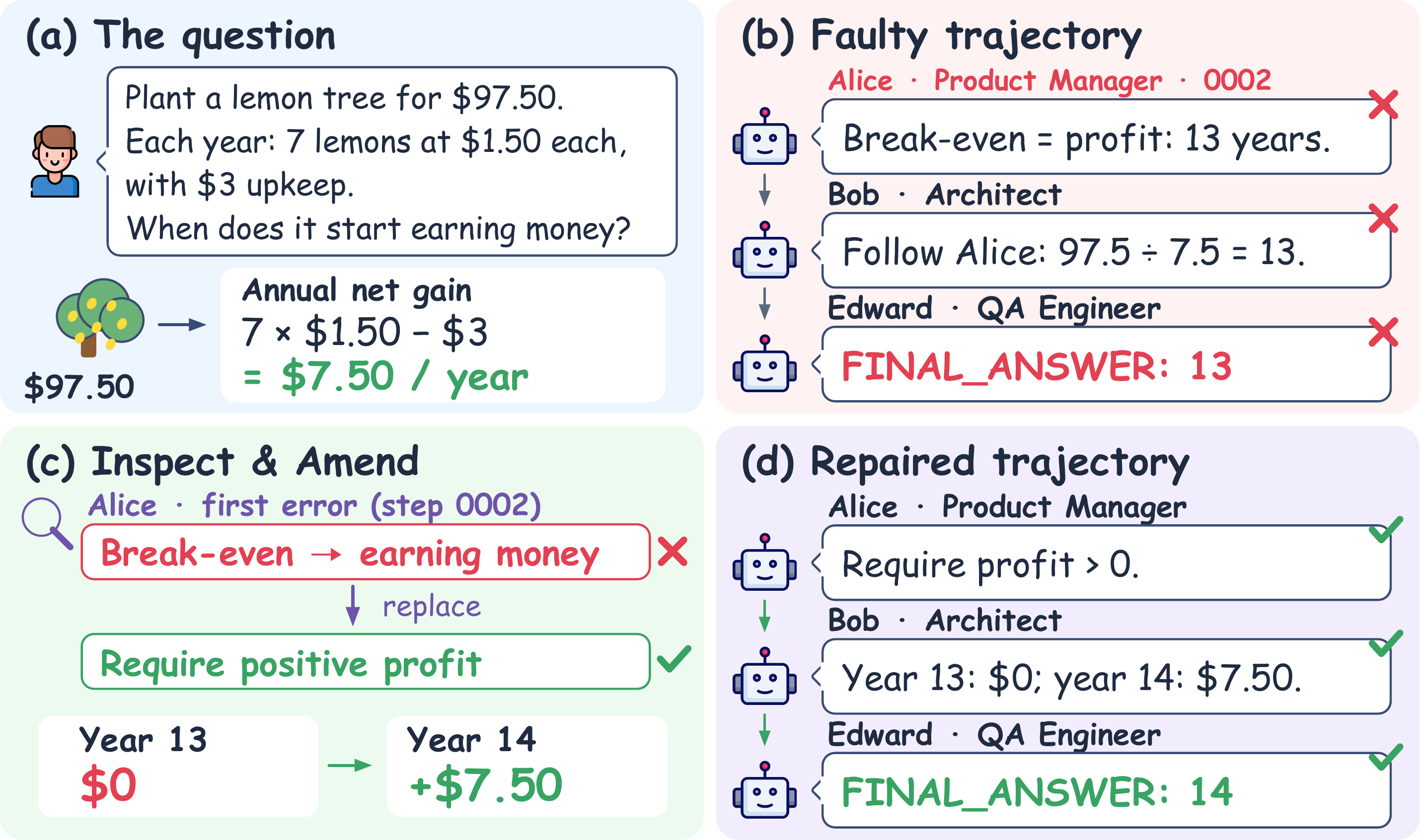}
\caption{An illustrative ARIA repair: localizing a task misunderstanding and patching the faulty step restores the correct result.}
\label{fig:repair-case}
\vspace{-0.7em}
\end{wrapfigure}

\paragraph{From failure analysis to repair supervision.}
Multi-agent systems (MAS) exhibit diverse failures, ranging from instruction violations and reasoning errors to inter-agent misalignment and inadequate verification~\citep{cemri2025mast,jia2026masfire}.
Existing datasets, such as MAST-Data and Who\&When, collect and annotate execution trajectories to characterize failure modes or identify the agents and steps responsible for unsuccessful outcomes~\citep{cemri2025mast,zhang2025whichagent}.
However, these failure-analysis datasets do not pair failed trajectories with executable repairs and benchmark-verified successful continuations, leaving a gap between understanding failures and learning how to recover from them.
While intervention-based debugging has demonstrated that failures can be repaired through targeted modifications and re-execution~\citep{ma2026dover}, our objective is to turn such recoveries into reusable training supervision.
We therefore construct a dataset that pairs failed trajectories with execution-verified repairs.
Its construction addresses two connected questions: which source configurations yield the most usable repair examples, and how should repairs be represented so that their effectiveness can be verified rather than merely asserted?

\paragraph{Repair-optimality for source selection.}
The most useful source of repair supervision is not necessarily the system that fails most often or the one with the strongest backbone; it is the one that produces the most successfully repairable failures per execution.
Let a source cell be a triple $c=(A,B,M)$, where $A$ is a multi-agent system with a fixed harness, $B$ is a benchmark with a fixed task distribution, and $M$ is the foundation model that the agents call.
Model selection must account for the interaction between these components: the same model can induce different failure distributions under different harness and benchmark combinations, and fault-injection measurements show that stronger foundation models do not uniformly improve robustness and that coordination topology is an equally decisive factor~\citep{jia2026masfire}.
We therefore compare candidate models within each fixed pair $(A,B)$, using a fixed repair protocol, rather than ranking them globally.
The following principle captures the trade-off underlying this selection.

\begin{definition}[Repair-optimality principle]
\label{def:repair-optimality}
For a fixed MAS--benchmark pair $(A,B)$, model selection balances two
quantities: failure availability, which determines the supply of candidate
repair instances, and conditional repairability, which determines how often
those failures yield benchmark-verified repairs under a fixed repair protocol.
A model that rarely fails provides too few repair instances, whereas a model
whose failures are rarely repairable provides little usable supervision.
\end{definition}

\noindent Definition~\ref{def:repair-optimality} states this trade-off qualitatively; the following proposition, proved in Appendix~\ref{app:repair-utility-proof}, makes it operational by identifying the single quantity that governs how much supervision a cell returns per execution.

\begin{proposition}[Repair utility score]
\label{prop:repair-utility-score}
For a cell $c$, let $a(c)$ be the probability that a valid source execution succeeds, and let $r(c)$ be the conditional probability that an attempted repair of a failed source execution succeeds under the benchmark evaluator. Define
\begin{equation}
\label{eq:repair-utility}
S(c) := (1-a(c))\,r(c).
\end{equation}
For $N$ independent valid source executions drawn from the same task distribution, the expected number of successfully repaired failures is $N S(c)$.
\end{proposition}

\noindent Equation~\ref{eq:repair-utility} thus measures the expected yield of verified repair examples, rather than failure frequency or source-task accuracy alone.
Over a finite candidate set $\mathcal{M}$, we select $M(A,B)\in\arg\max_{M\in\mathcal{M}}S(A,B,M)$ independently for each pair in the experimental grid.
This rule identifies repair-optimal source cells without imposing a global ordering over models.
We estimate $a(c)$ from valid source executions and $r(c)$ as the fraction of attempted repairs that yield a verified success under the fixed repair protocol.
Invalid infrastructure runs are excluded from source accuracy, while unresolved repair attempts are reported as a separate status rather than as incorrect task solutions; they do not count as verified repairs.
This distinction prevents infrastructure failures from inflating failure availability and unresolved attempts from inflating verified repair yield.

\paragraph{From failure taxonomies to repair attributions.}
Failure taxonomies describe errors but do not specify how to correct them~\citep{cemri2025mast}.
We introduce \emph{repair attributions} (RA), comprising six groups and 26 categories that link localized failure mechanisms to candidate interventions.
Each attribution guides where to cut the trajectory and how to replace its continuation.
We retain repair supervision only when paired executions confirm that the patched run succeeds while the unpatched control fails.
The attributions then serve as training targets and conditioning information for repair learning.
Appendix~\ref{app:repair-attributions} provides the full taxonomy, annotation criteria, intervention examples, and validation details.

% =========================================================================
\subsection{Post-Training for Generalizable Repair}
\label{subsec:post-training}

\begin{figure}[!t]
\centering
\includegraphics[width=\linewidth]{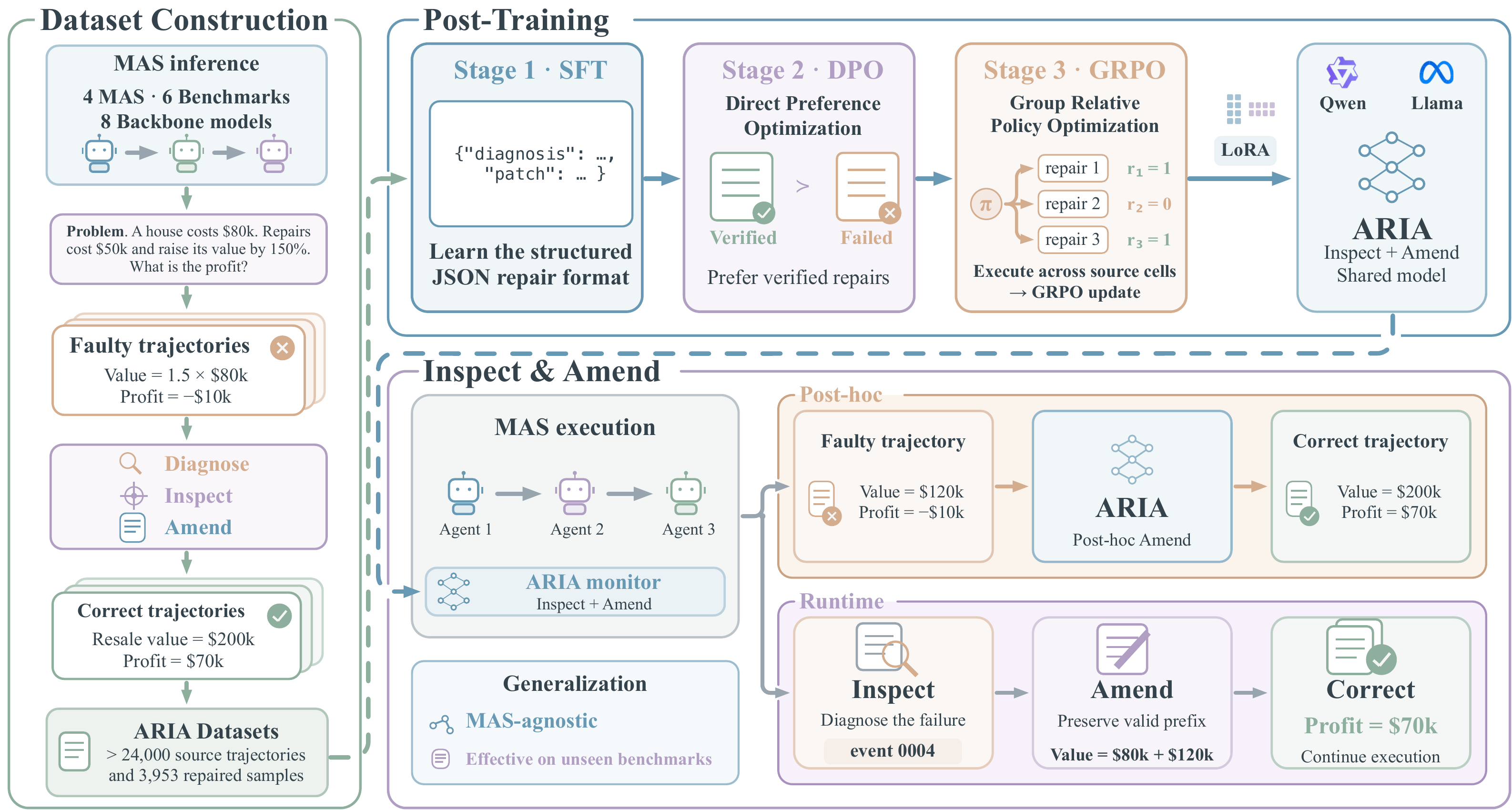}
\caption{Overview of ARIA, which constructs execution-verified repair data, learns generalizable repair strategies through three-stage post-training, and applies them at runtime to inspect and amend failing multi-agent trajectories.}
\label{fig:aria-framework}
\end{figure}

We learn structured repairs $y_i=(d_i,k_i,e_i)$ from
$\mathcal{D}=\{(x_i,y_i)\}_{i=1}^{N}$ in three stages:
structured supervision, outcome-anchored preference learning, and
execution-grounded policy optimization. The first two stages use
corpus-derived supervision; the third evaluates current-policy samples
to address distribution shift~\citep{ross2011reduction} and align training
with task recovery and intervention costs. Appendix~\ref{app:post-training}
provides the stage overview and implementation details.

\paragraph{Stage 1: Structured Repair SFT.}
\label{subsubsec:stage1}
We first learn the output schema from records with execution-verified
repairs. Under constrained decoding, we factorize the supervised loss as
\begin{equation}
\label{eq:sft}
\begin{aligned}
\mathcal{L}_{\mathrm{SFT}}(\theta)
= -\mathbb{E}_{(x_i,y_i)\sim\mathcal{D}}\big[
&\log p_\theta(d_i\mid\tilde{x}_i)
+\log p_\theta(k_i\mid\tilde{x}_i,d_i)+\log p_\theta(e_i\mid\tilde{x}_i,d_i,k_i)
\big].
\end{aligned}
\end{equation}
Feedback dropout in $\tilde{x}_i$ supports both feedback-conditioned
post-hoc repair and feedback-free runtime repair. We jointly train an
inspection gate $g_\theta(\tau_{i,\leq t})$ to detect whether the certified
failure has occurred, using $z_{i,t}=\mathbb{1}[t>k_i]$ for validated cuts
and successful repaired prefixes as negatives. Thus, Stage 1 initializes
both structured repair and runtime inspection; its checkpoint $\theta_0$
also serves as the Stage 2 reference policy.

\paragraph{Stage 2: Outcome-Anchored Preference Learning.}
\label{subsubsec:stage2}
Verified demonstrations do not explicitly rank repairs against failed
alternatives. We therefore prefer the verified repair $a^+=(k_i,e_i)$
to the observed failed continuation $a^-=(k_i,\tau_{i,>k_i})$ under
the common context $h=(\tilde{x}_i,\hat d)$; both retain the prefix
$\tau_{i,\leq k_i}$. Writing
$\Delta_\theta(u\mid h)=\log[p_\theta(u\mid h)/p_{\theta_0}(u\mid h)]$,
we minimize the DPO loss~\citep{rafailov2023dpo}:
\begin{equation}
\label{eq:pt-dpo}
\mathcal{L}_{\mathrm{DPO}}(\theta)
=-\mathbb{E}_{(h,a^+,a^-)}
\log\sigma\!\left(
\beta\left[\Delta_\theta(a^+\mid h)-\Delta_\theta(a^-\mid h)\right]
\right),
\end{equation}
where $\sigma$ is the logistic sigmoid and $\beta$ controls reference
regularization. This anchors repair preferences to verified outcome
contrasts.
Diagnoses lack direct outcome labels, so we favor those that improve
repair prediction for related training incidents. For a candidate $d$,
we measure its uplift on calibration siblings
$F_{\mathrm{cal}}(i)\subset F(i)$:
\begin{equation}
\label{eq:uplift}
U_i(d)=\frac{1}{|F_{\mathrm{cal}}(i)|}
\sum_{j\in F_{\mathrm{cal}}(i)}
\left[
\log p_{\bar\theta}(k_j,e_j\mid x_j,d)
-\log p_{\bar\theta}(k_j,e_j\mid x_j)
\right],
\end{equation}
where $F(i)$ groups related incidents and $\bar\theta$ is a frozen scorer
initialized from $\theta_0$. This is a likelihood-based transfer proxy,
not executed repair success. High-uplift diagnoses that also improve a
held-out training sibling are preferred to low-uplift candidates using
the same DPO objective. We then internalize the selected diagnoses
alongside verified repairs through supervised learning, yielding
$\theta_1$. Family construction, selection criteria, and internalization
are detailed in Appendix~\ref{app:pt-stage2}.

\paragraph{Stage 3: Transfer-Grounded Policy Optimization.}
\label{subsubsec:stage3}
We initialize $\pi_\theta$ at $\theta_1$ and evaluate its samples
through re-execution. For a statically valid record
$y=(\hat d,\hat k,\hat e)$, the environment of
Section~\ref{subsec:dataset-creation} restores the state at $\hat k$,
substitutes $\hat e$, and resumes execution, returning
\begin{equation}
\label{eq:stage3-reward}
r(y\mid x)
=\mathbb{1}[\text{evaluator accepts}]
-\lambda_h\,\mathbb{1}[\text{harmful}]
-\lambda_c\,b(y),
\end{equation}
where $b(y)$ is execution overhead. Unlike the Stage 2 proxy, this reward
measures task recovery and intervention costs through execution.
A diagnosis is not executable, so we credit it through the repairs it
induces on its source incident and on related incidents from other
source cells. For each sampled diagnosis $d_q$, we generate $n$ source
repair records $y_{q,j}$ and $n_t$ sibling repair records $y'_{q,l}$,
all conditioned on $d_q$, and define
\begin{equation}
\label{eq:diagnosis-value}
\begin{aligned}
V(d_q)
={}&(1-\alpha)\frac{1}{n}\sum_{j=1}^{n}
 r(y_{q,j}\mid\tilde{x}_i)+\alpha\frac{1}{n_t}\sum_{l=1}^{n_t}
 r(y'_{q,l}\mid\tilde{x}_{\nu_l(i)}).
\end{aligned}
\end{equation}
Here $\nu_l(i)\in F(i)$ satisfies $c(\nu_l(i))\neq c(i)$, where $c$
indexes the framework--benchmark--host-model configuration, and $\alpha$
weights cross-cell transfer. Diagnosis and repair tokens receive
separately normalized GRPO advantages~\citep{shao2024deepseekmath}, with
regularization toward $\pi_{\theta_1}$. This favors diagnoses with
executed cross-cell utility while optimizing repairs within each
diagnosis.

Finally, the inspection gate must balance recovering failed runs against
damaging runs that would otherwise succeed. With repair parameters frozen,
we optimize complete-task rollouts using the matched-seed reward
\begin{equation}
\label{eq:gate-reward}
r_{\mathrm{gate}}
=S_{\mathrm{int}}-S_{\mathrm{base}}-\lambda_c\,\Delta\mathrm{EB},
\end{equation}
where $S_{\mathrm{int}},S_{\mathrm{base}}\in\{0,1\}$ indicate task success
with and without gate-controlled intervention under the same seed, and
$\Delta\mathrm{EB}$ is the execution-overhead difference. The outcome
difference rewards recovered failures and penalizes harmful
interventions relative to the untreated baseline. Alternating repair
and gate optimization yields $\theta_2$, evaluated in
Section~\ref{sec:experiment}; sampling, stabilization, and validation
protocols are given in Appendix~\ref{app:pt-stage3}.

% =========================================================================
% Experiments
% =========================================================================

\section{Experiments}
\label{sec:experiment}

Our experiments examine ARIA's repair effectiveness, generalization, and the contributions of its post-training stages through three questions.
\textbf{RQ1:} Can ARIA recover failed executions more effectively and at lower execution cost than existing repair methods in both the post-hoc and runtime intervention settings of Section~\ref{subsec:post-training}?
\textbf{RQ2:} How well does ARIA generalize to configurations withheld from training, and how does its performance vary with backbone size?
\textbf{RQ3:} What does each of the three post-training stages contribute to ARIA's performance?

% =========================================================================
\subsection{Experimental Setup}
\label{subsec:exp-setup}

\paragraph{Backbone models.}
\emph{Host models} power the executing agents, determine the source failure distribution, and are selected per cell by Equation~\ref{eq:repair-utility}: Llama-3.1-70B-Instruct~\citep{llama2024herd}, Qwen3-8B/32B, GPT-5.5 (M/H/X reasoning tiers), Claude-Sonnet-4-6, and Claude-Opus-4.8. Only the disjoint \emph{ARIA backbones} are post-trained: Qwen2.5-Instruct at 1.5B, 3B, 7B, and 14B~\citep{qwen2025qwen25} for scaling under a shared pre-training and post-training recipe, and Llama-3.1-8B-Instruct~\citep{llama2024herd} as a cross-family control comparable in size to Qwen2.5-7B. Appendix~\ref{app:reproduction} gives the controls used for backbone comparisons.

\paragraph{Multi-agent systems.}
We use four coordination structures in the in-domain experiments: MetaGPT's fixed role pipeline~\citep{hong2024metagpt}, OpenManus's tool use and delegation~\citep{openmanus2025}, Magentic-One's ledger-based orchestration~\citep{fourney2024magenticone}, and AG2's programmable group chats~\citep{wu2023autogenenablingnextgenllm}. RQ2 additionally treats Magentic-One as an unseen system by retraining ARIA without any of its trajectories. Frameworks run on supported benchmarks; each reported cell identifies its framework, benchmark, and host model.

\paragraph{Benchmarks and evaluation sets.}
Six benchmarks span four categories: mathematics through GSM-PLUS~\citep{li2024gsmplus} and OlympiadBench~\citep{he2024olympiadbench}; knowledge through MMLU~\citep{hendrycks2021mmlu}; code synthesis and repository repair through HumanEval~\citep{chen2021humaneval} and SWE-bench-Lite~\citep{jimenez2024swebench}; and tool use through GAIA~\citep{mialon2024gaia}. RQ1 uses an \emph{in-domain} test split, while RQ2 compares in-domain and \emph{unseen system} evaluation across backbone scales. Appendix~\ref{app:reproduction} details data partitioning, held-out-set construction, and fitting restrictions.

\paragraph{Baselines.}
Post-hoc comparisons include zero-shot repair; the failure-localization methods Who\&When~\citep{zhang2025whichagent}, AgentRx~\citep{barke2026agentrx}, and ErrorProbe~\citep{li2026errorprobe}; and the intervention-validation method DoVer~\citep{ma2026dover}. Runtime comparisons include Spark to Fire~\citep{xie2026sparktofire}, which controls inter-agent messages, and SafeAgent~\citep{liu2026safeagent}, an action-mediation method included as a safety-oriented reference. Appendix~\ref{app:reproduction} gives the baseline inputs, adaptation protocol, model controls, and execution accounting.

\paragraph{Metrics.}
Success Rate (SR) measures benchmark passes among failed test trajectories in the post-hoc setting or among all test tasks in the runtime setting. Harmful Intervention Rate (HIR) measures interventions that are structurally invalid, violate task constraints, or reduce the graded benchmark score relative to no intervention. Execution Budget (EB) is the average number of full multi-agent executions per evaluated task.

% =========================================================================
% Page 7
\subsection{Repair Effectiveness and Execution Cost (RQ1)}
\label{subsec:rq1-effectiveness}

We evaluate ARIA on the in-domain test set in two complementary settings: post-hoc repair of failed trajectories and runtime amendment of ongoing executions.

\begin{table}[!t]
\centering
\footnotesize
\renewcommand{\arraystretch}{0.92}
\setlength{\tabcolsep}{3.5pt}
\caption{Post-hoc repair performance on in-domain failed test trajectories with evaluator feedback, pooled across four frameworks. We report success rate (SR; higher is better) and execution budget (EB; lower is better), with SR additionally broken down by benchmark. The best value in each comparison block is shown in \textbf{bold}.}
\vspace{-1em}
\label{tab:rq1-posthoc}
\begin{tabular*}{\linewidth}{@{\extracolsep{\fill}}lcccccccc@{}}
\toprule
& \multicolumn{6}{c}{\textbf{SR by benchmark} $\uparrow$} & \multicolumn{2}{c}{\textbf{Overall}} \\
\cmidrule(lr){2-7} \cmidrule(lr){8-9}
\textbf{Method} & GSM-P & MMLU & HEval & Olymp & GAIA & SWE-L & SR $\uparrow$ & EB $\downarrow$ \\
\midrule
\rowcolor{tabband}\multicolumn{9}{@{}c@{}}{\emph{Llama-3.1-8B backbone}} \\
Zero-shot & 7.37\% & 10.5\% & 2.63\% & 2.31\% & 0.00\% & 1.54\% & 8.37\% & 1.98 \\
Who\&When~\citeyearpar{zhang2025whichagent} & 12.4\% & 14.7\% & 0.00\% & 1.39\% & 1.92\% & 3.08\% & 12.5\% & 1.98 \\
AgentRx~\citeyearpar{barke2026agentrx} & 8.57\% & 8.46\% & 2.63\% & 3.70\% & 1.92\% & 1.54\% & 7.99\% & \textbf{1.97} \\
ErrorProbe~\citeyearpar{li2026errorprobe} & 6.82\% & 7.42\% & 2.63\% & 1.39\% & 0.00\% & 0.00\% & 6.58\% & 1.98 \\
DoVer~\citeyearpar{ma2026dover} & 1.45\% & 8.36\% & 13.2\% & 8.80\% & 3.85\% & 3.08\% & 5.39\% & 5.75 \\
\textbf{ARIA} & \textbf{16.8\%} & \textbf{18.5\%} & \textbf{15.8\%} & \textbf{9.72\%} & \textbf{7.69\%} & \textbf{4.62\%} & \textbf{16.9\%} & 2.00 \\
\midrule
\rowcolor{tabband}\multicolumn{9}{@{}c@{}}{\emph{Qwen2.5-7B backbone}} \\
Zero-shot & 25.1\% & 21.4\% & 5.26\% & 2.31\% & 1.92\% & 1.54\% & 21.2\% & 1.98 \\
Who\&When~\citeyearpar{zhang2025whichagent} & 43.5\% & 26.8\% & 7.89\% & 2.78\% & 3.85\% & 3.08\% & 31.6\% & 1.98 \\
AgentRx~\citeyearpar{barke2026agentrx} & 29.3\% & 21.9\% & 2.63\% & 4.63\% & 3.85\% & 0.00\% & 23.3\% & \textbf{1.96} \\
ErrorProbe~\citeyearpar{li2026errorprobe} & 35.4\% & 22.0\% & 2.63\% & 2.78\% & 1.92\% & 1.54\% & 25.8\% & 1.97 \\
DoVer~\citeyearpar{ma2026dover} & 7.00\% & 4.83\% & 18.4\% & 1.85\% & 5.77\% & 3.08\% & 5.69\% & 3.24 \\
\textbf{ARIA} & \textbf{47.0\%} & \textbf{30.5\%} & \textbf{21.1\%} & \textbf{6.94\%} & \textbf{7.69\%} & \textbf{6.15\%} & \textbf{35.3\%} & 2.00 \\
\bottomrule
\end{tabular*}
\vspace{-1em}
\end{table}

\paragraph{Post-hoc Repair.}
In the post-hoc setting, each method receives a complete failed trajectory together with evaluator feedback and task constraints, and the corresponding task evaluator grades the repaired continuation. Table~\ref{tab:rq1-posthoc} reveals two consistent patterns. First, ARIA achieves the highest overall SR with both backbones and leads on every benchmark, spanning the large mathematics and knowledge pools as well as the smaller code, tool-use, and repository-repair sets. The localization baselines are strongest on only a subset of reasoning tasks, while DoVer is competitive on HumanEval and OlympiadBench but substantially weaker in aggregate. ARIA's lead also persists from Llama to Qwen, showing that post-training adds value beyond the backbone's stronger zero-shot repair ability.
Second, the gain does not come from additional execution. ARIA's average EB remains essentially the same as the zero-shot and localization baselines, whereas DoVer requires more replay to validate interventions through re-execution. In the matched-budget sweep in Figure~\ref{fig:rq1-execution-budget}, ARIA leads at the smallest tested budget and improves steadily; the alternatives remain lower and grow more slowly. Additional budget therefore amplifies an already effective repair policy rather than compensating for weak initial proposals.

\begin{table}[!b]
\centering
\footnotesize
\renewcommand{\arraystretch}{0.92}
\setlength{\tabcolsep}{5pt}
\caption{Runtime intervention performance on all in-domain test tasks without evaluator feedback. We report per-framework and overall SR, HIR, and EB; HIR is measured against uninterrupted runs. The best value in each comparison block is shown in \textbf{bold}.}
\label{tab:rq1-runtime}
\begin{tabular*}{\linewidth}{@{\extracolsep{\fill}}lccccccc@{}}
\toprule
& \multicolumn{4}{c}{\textbf{SR by framework} $\uparrow$} & \multicolumn{3}{c}{\textbf{Overall}} \\
\cmidrule(lr){2-5} \cmidrule(lr){6-8}
\textbf{Method} & OpenManus & MetaGPT & Magentic-One & AG2 & SR $\uparrow$ & HIR $\downarrow$ & EB $\downarrow$ \\
\midrule
\rowcolor{tabband}\multicolumn{8}{@{}c@{}}{\emph{Llama-3.1-8B backbone}} \\
Spark to Fire~\citeyearpar{xie2026sparktofire} & 17.0\% & 7.31\% & \textbf{2.99\%} & 6.39\% & 12.5\% & 7.51\% & {1.00} \\
SafeAgent~\citeyearpar{liu2026safeagent} & 20.3\% & 20.6\% & 0.80\% & 20.9\% & 18.4\% & 5.50\% & {1.00} \\
\textbf{ARIA} & \textbf{20.8\%} & \textbf{20.8\%} & 1.20\% & \textbf{23.7\%} & \textbf{19.2\%} & \textbf{1.49\%} & {1.00} \\
\midrule
\rowcolor{tabband}\multicolumn{8}{@{}c@{}}{\emph{Qwen2.5-7B backbone}} \\
Spark to Fire~\citeyearpar{xie2026sparktofire} & 18.6\% & 8.75\% & 1.60\% & 5.70\% & 13.5\% & 8.00\% & {1.00} \\
SafeAgent~\citeyearpar{liu2026safeagent} & 19.6\% & 18.9\% & 1.20\% & 21.8\% & 17.8\% & 13.9\% & {1.00} \\
\textbf{ARIA} & \textbf{23.0\%} & \textbf{23.5\%} & \textbf{4.59\%} & \textbf{24.5\%} & \textbf{21.4\%} & \textbf{2.00\%} & {1.00} \\
\bottomrule
\end{tabular*}
\end{table}

\paragraph{Runtime Amendment.}
In the runtime setting, each method inspects ongoing executions without evaluator feedback. Table~\ref{tab:rq1-runtime} shows that ARIA achieves the strongest overall success--harm trade-off with both backbones: its higher aggregate SR is accompanied by markedly lower HIR, rather than by more aggressive intervention. The gains are also consistent across coordination structures, with ARIA leading on OpenManus, MetaGPT, and AG2. The only per-framework exception is Magentic-One with the Llama backbone, yet ARIA still retains the best aggregate SR and lowest HIR in that comparison block. Moving to the stronger Qwen backbone further widens ARIA's aggregate advantage, suggesting that additional capacity benefits both the decision to intervene and the construction of an effective amendment. Because all methods operate under the same single-execution budget, these improvements reflect a better runtime policy rather than replay-based selection.
Figure~\ref{fig:rq1-sr-hir} further shows that the 1.5B model already improves over zero-shot amendment and approaches the specialized baselines, while the 7B and 14B models shift decisively toward the upper-left region on both frameworks, simultaneously raising SR and reducing HIR. Scaling therefore improves gate selectivity as well as repair effectiveness. In Figure~\ref{fig:rq1-execution-budget}, ARIA leads at the lowest budget and remains ahead as the budget grows, confirming that its advantage appears before repeated trials become available and persists when additional execution is allowed.

% =========================================================================
% Page 8

\begin{figure}[!t]
\centering
\begin{subfigure}[t]{0.49\linewidth}
\centering
\vspace{0pt}
\includegraphics[width=\linewidth]{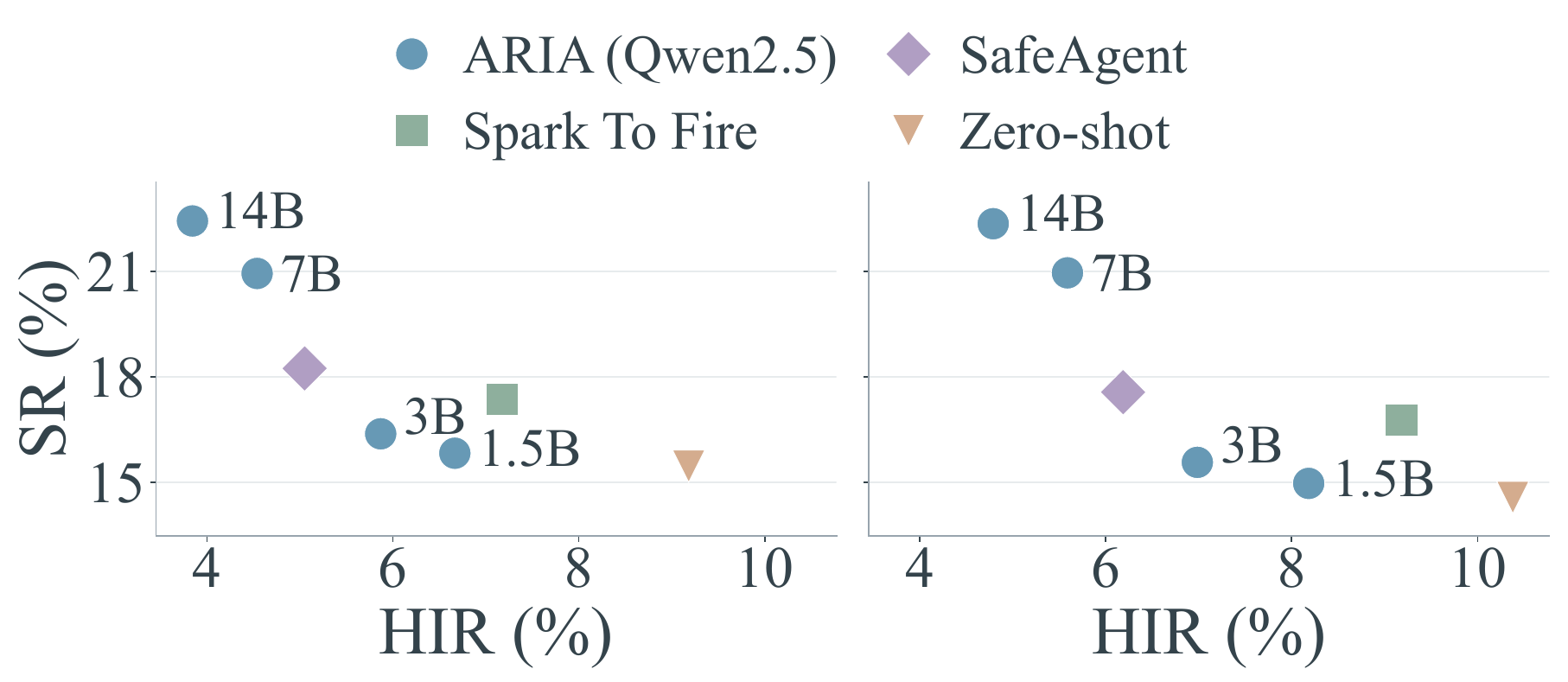}
\caption{Runtime success--harm trade-offs are shown for OpenManus (left) and MetaGPT (right).}
\label{fig:rq1-sr-hir}
\end{subfigure}\hfill
\begin{subfigure}[t]{0.49\linewidth}
\centering
\vspace{0pt}
\includegraphics[width=\linewidth]{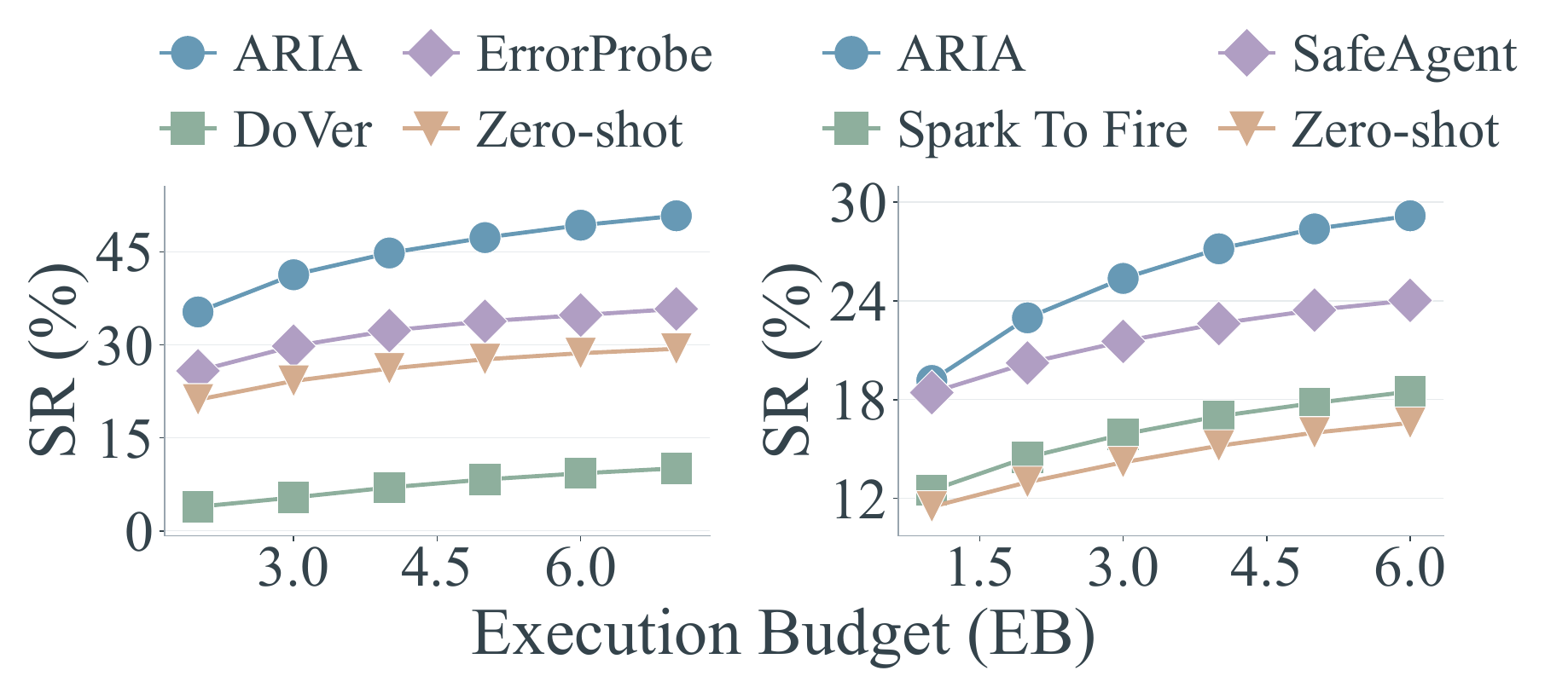}
\caption{Success rates under varying EB are shown for post-hoc repair (left) and runtime amendment (right).}
\label{fig:rq1-execution-budget}
\end{subfigure}
\caption{ARIA's repair effectiveness and execution cost.}
\label{fig:rq1-overview}
\end{figure}

% \clearpage

\begin{table}[!t]
\centering
\footnotesize
\renewcommand{\arraystretch}{0.95}
\setlength{\tabcolsep}{6pt}
\caption{Generalization performance across backbone scales in the \emph{in-domain} and \emph{unseen system} (Magentic-One trajectories excluded from training and validation) settings. We report SR and HIR.}
\label{tab:rq2-generalization}
\begin{tabular*}{\linewidth}{@{\extracolsep{\fill}}llcccc@{}}
\toprule
& & \multicolumn{2}{c}{\textbf{In-domain}} & \multicolumn{2}{c}{\textbf{Unseen System}} \\
\cmidrule(lr){3-4} \cmidrule(lr){5-6}
\textbf{Method} & \textbf{Backbone} & SR $\uparrow$ & HIR $\downarrow$ & SR $\uparrow$ & HIR $\downarrow$ \\
\midrule
\rowcolor{tabband}\multicolumn{6}{@{}c@{}}{\emph{Reference baselines}} \\
Spark to Fire~\citeyearpar{xie2026sparktofire} & Qwen2.5-7B & 17.4\% & 7.18\% & 16.8\% & 9.18\% \\
SafeAgent~\citeyearpar{liu2026safeagent} & Qwen2.5-7B & 18.2\% & 5.05\% & 17.6\% & 6.19\% \\
\midrule
\rowcolor{tabband}\multicolumn{6}{@{}c@{}}{\emph{Backbone scale}} \\
\textbf{ARIA} & Qwen2.5-1.5B & 15.8\% & 6.67\% & 15.0\% & 8.18\% \\
\textbf{ARIA} & Qwen2.5-3B & 16.4\% & 5.87\% & 15.6\% & 6.99\% \\
\textbf{ARIA} & Qwen2.5-7B & 20.9\% & 4.54\% & 21.0\% & 5.59\% \\
\textbf{ARIA} & Qwen2.5-14B & 22.4\% & 3.84\% & 22.4\% & 4.79\% \\
\midrule
\rowcolor{tabband}\multicolumn{6}{@{}c@{}}{\emph{Cross-family control}} \\
\textbf{ARIA} & Llama-3.1-8B & 20.6\% & 4.81\% & 20.6\% & 5.79\% \\
\bottomrule
\end{tabular*}
\vspace{-1em}
\end{table}

% =========================================================================
\subsection{Generalization and Scaling (RQ2)}
\label{subsec:rq2-generalization}

All RQ2 experiments use the runtime setting and the same post-training recipe across Qwen2.5 backbones from 1.5B to 14B. For \emph{unseen system} evaluation, we exclude every Magentic-One trajectory from training, validation, and checkpoint selection, then deploy the resulting model directly on Magentic-One. Success therefore requires transferring both the inspection decision and the amendment strategy rather than recalling system-specific traces.
Table~\ref{tab:rq2-generalization} shows that unseen-system SR closely tracks in-domain performance at every scale, while HIR increases only modestly. The distribution shift therefore affects intervention confidence more than the model's ability to produce useful repairs. More importantly, the model-size ordering is unchanged: larger backbones monotonically improve both recovery and selectivity, so scaling does not exchange higher success for more aggressive intervention. The small variants remain near the runtime baselines and do not yet dominate them consistently. At 7B, however, ARIA moves beyond the strongest baseline on both metrics, and the 14B model pushes the success--harm frontier further outward. This change suggests that sufficient capacity is especially useful for jointly representing the execution state, deciding whether intervention is warranted, and constructing a compatible continuation. A similarly sized Llama control closely follows Qwen2.5-7B, indicating that the effect transfers across backbone families rather than depending on Qwen-specific behavior.
Figure~\ref{fig:rq2-generalization} tests the trend under two distinct held-out coordination structures. Every increase in backbone size improves held-out SR, and neither system reverses the in-domain scale ordering. The held-out curves also rise approximately in parallel with their in-domain counterparts, showing that the benefit of additional capacity survives the system shift rather than being confined to familiar trajectories. AG2 remains closer to its in-domain curve, whereas Magentic-One exhibits the larger transfer gap and is consequently the harder setting.

% =========================================================================
% Page 9
\begin{figure}[!t]
\centering
\begin{subfigure}[t]{0.51\linewidth}
\centering
\vspace{0pt}
\includegraphics[width=\linewidth]{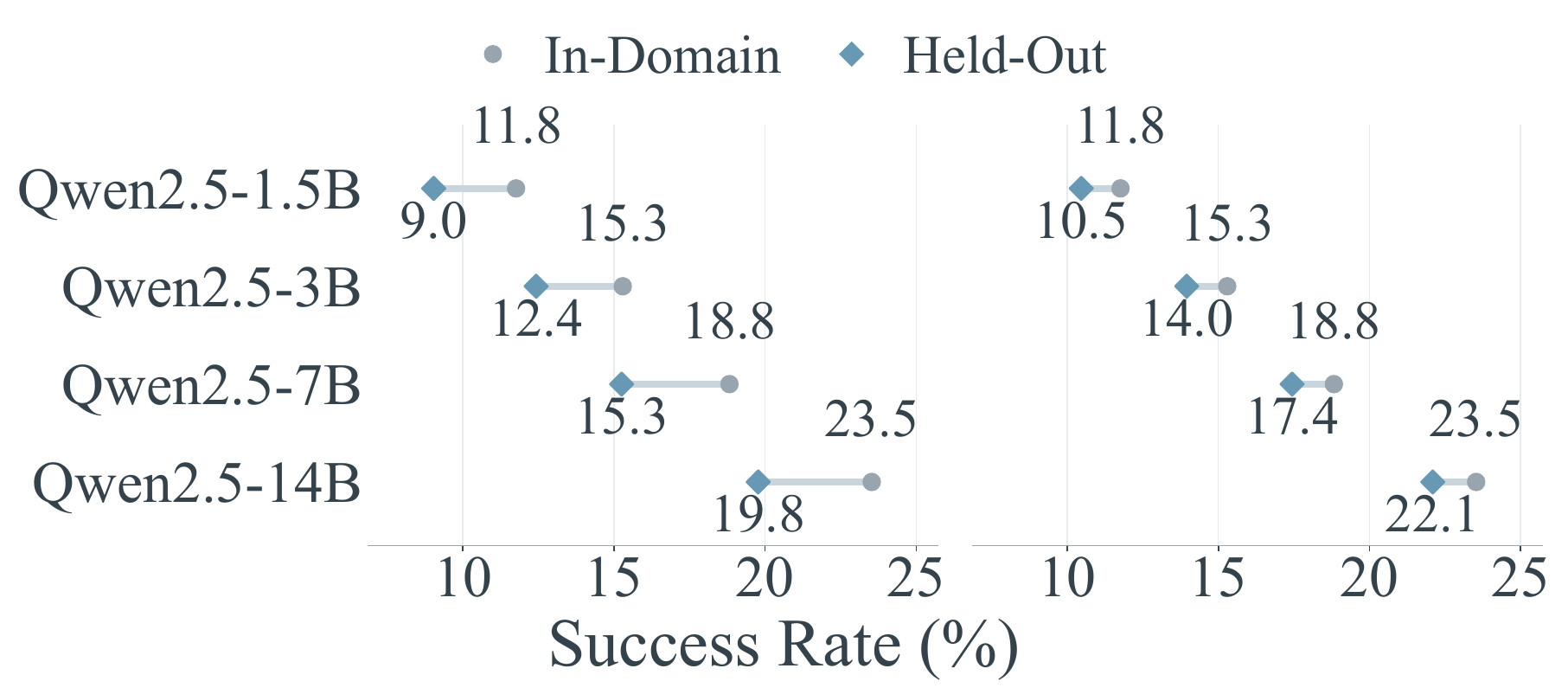}
\caption{Runtime generalization across Qwen2.5 backbone scales; left: \emph{unseen system} Magentic-One, right: \emph{unseen system} AG2.}
\label{fig:rq2-generalization}
\end{subfigure}\hfill
\begin{subfigure}[t]{0.48\linewidth}
\centering
\vspace{0pt}
\includegraphics[width=\linewidth]{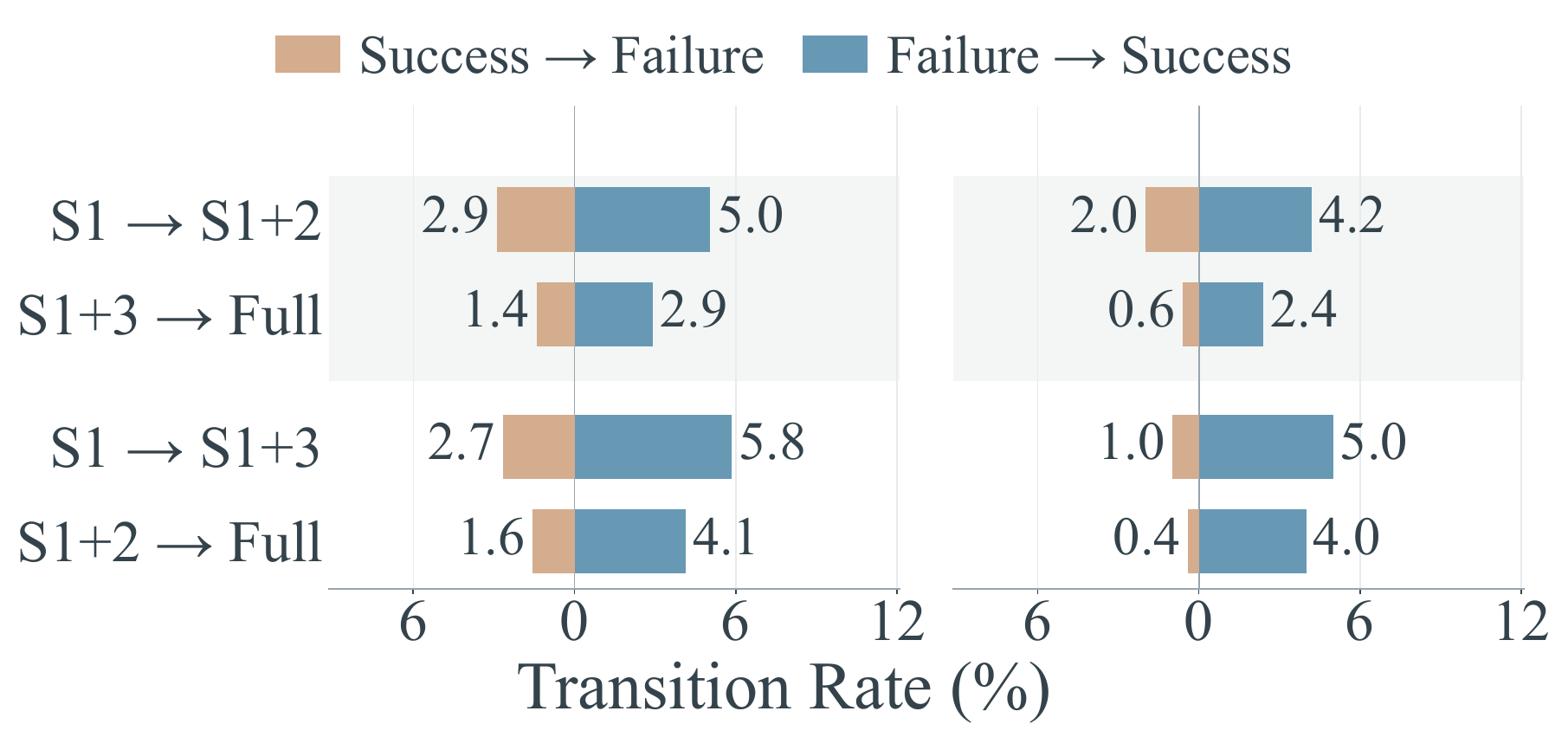}
\caption{Outcome transitions across ARIA post-training variants; left: in-domain tasks, right: \emph{unseen system} Magentic-One tasks.}
\label{fig:rq3-ablation-transitions}
\end{subfigure}
\caption{ARIA's generalization ability and the contributions of its post-training stages.}
\label{fig:rq2-rq3-overview}
\end{figure}

\begin{table}[!t]
% \vspace{-1em}
\centering
\footnotesize
\renewcommand{\arraystretch}{0.95}
\setlength{\tabcolsep}{8pt}
\caption{Ablation of ARIA's post-training stages with the Qwen2.5-7B backbone. We report SR and HIR; the best result in each column is shown in \textbf{bold}.}
\label{tab:rq3-ablation}
\begin{tabular*}{\linewidth}{@{\extracolsep{\fill}}lcccc@{}}
\toprule
& \multicolumn{2}{c}{\textbf{In-domain}} & \multicolumn{2}{c}{\textbf{Unseen System}} \\
\cmidrule(lr){2-3} \cmidrule(lr){4-5}
\textbf{Variant} & SR $\uparrow$ & HIR $\downarrow$ & SR $\uparrow$ & HIR $\downarrow$ \\
\midrule
\rowcolor{tabband}\multicolumn{5}{@{}c@{}}{\emph{Reference baselines}} \\
Zero-shot & 11.7\% & 10.3\% & 9.78\% & 11.4\% \\
\midrule
\rowcolor{tabband}\multicolumn{5}{@{}c@{}}{\emph{ARIA post-training variants}} \\
Stage 1 only (w/o Stage 2 and Stage 3) & 14.5\% & 8.22\% & 12.4\% & 8.78\% \\
Stage 1 + 2 (w/o Stage 3) & 16.6\% & 4.15\% & 14.6\% & 5.19\% \\
Stage 1 + 3 (w/o Stage 2) & 17.6\% & 3.13\% & 16.4\% & 3.79\% \\
\textbf{ARIA (Stage 1 + 2 + 3)} & \textbf{19.2\%} & \textbf{1.49\%} & \textbf{18.2\%} & \textbf{2.20\%} \\
\bottomrule
\end{tabular*}
\vspace{-1em}
\end{table}

% =========================================================================
\subsection{Ablation of the Post-Training Stages (RQ3)}
\label{subsec:rq3-ablation}

Using Qwen2.5-7B, we examine how the three post-training stages shape runtime repair. As summarized in Table~\ref{tab:rq3-ablation}, Stage 1 establishes a useful structured repair policy and inspection gate; adding Stage 2 primarily reduces harmful interventions, indicating that preference contrasts sharpen when the model should act, whereas Stage 3 produces the larger recovery gain by grounding proposed repairs in execution outcomes. The full pipeline improves on both partial variants in-domain and on the unseen system, so preference learning and execution optimization remain complementary under distribution shift.
Notably, the unseen-system gaps remain directionally similar, suggesting that each stage learns transferable behavior rather than merely fitting framework-specific intervention patterns.
Figure~\ref{fig:rq3-ablation-transitions} confirms that these aggregate gains are not caused by unstable task churn. Every transition to a richer variant converts more failures into successes than successes into failures in both settings. Stage 3 creates the largest positive flow from Stage 1 and remains beneficial after Stage 2, while Stage 2 still contributes net recoveries after Stage 3. Together, the stages progress from structured intervention, to selective intervention, to execution-grounded recovery. Overall, the ablation validates the superiority of ARIA's complete staged post-training design over zero-shot and partial variants in delivering effective, selective, and transferable runtime repair.

\section{Conclusion}
\label{sec:conclusion}

We introduced ARIA, a MAS-agnostic framework for detecting and repairing failures during multi-agent execution.
From over 24{,}000 trajectories, ARIA constructs 3{,}953 execution-verified repair pairs that specify the failure, intervention point, and corrected continuation.
Its three-stage training combines structured supervision, outcome-anchored preference learning, and execution-grounded policy optimization, moving costly repair search offline.
Across diverse frameworks, benchmarks, and host models, ARIA delivers strong post-hoc and runtime repair performance while controlling harmful interventions and execution cost.
It further transfers to unseen host models and systems, with performance improving as the repair backbone scales.
Ablations confirm that all three training stages contribute to reliable recovery.
Together, these results establish execution-verified repair learning as a promising route toward more reliable multi-agent systems.

\clearpage
\subsection*{AI Use Statement}
We used generative AI tools solely for language editing and polishing. We did not use generative AI for research ideation, methodology design, code development, data collection, experiments, result analysis, or scientific claims. All AI-assisted text was manually reviewed and verified by the authors. We take full responsibility for the final content of this work, including any text produced with the aid of generative AI.

\subsection*{Ethics Statement}
This work studies automated repair for LLM-based multi-agent systems using established benchmarks and machine-generated execution trajectories. It does not involve human participants, personal or sensitive data, or deployment in high-stakes real-world settings. Nevertheless, automated interventions may introduce errors or unintended behavior and should not be treated as a guarantee of system safety. ARIA should therefore be evaluated under task-specific constraints and appropriate human oversight before real-world deployment.

\bibliography{references}
\bibliographystyle{iclr2027_conference}

\clearpage
\appendix
% =========================================================================
% Related Work
% =========================================================================

\section{Related Work}
\label{sec:related-work}

\emph{Failure taxonomies.}
A first line of work asks which failures occur in multi-agent executions and how often. \citet{hammond2025multiagentrisks} organize risks from advanced multi-agent systems into three failure modes, namely miscoordination, conflict, and collusion, and seven risk factors that can produce them, and argue that these risks are qualitatively distinct from those posed by individual agents. Restricting attention to LLM-based MAS, \citet{cemri2025mast} develop MAST from 150 expert-annotated traces with high inter-annotator agreement, obtaining 14 failure modes grouped into system-design issues, inter-agent misalignment, and task verification, and apply an LLM-as-a-judge annotator to release MAST-Data, a corpus of more than 1{,}600 traces collected across 7 MAS frameworks. Their case studies show that taxonomy-guided revisions to role specifications and workflows raise task success, which locates the interventions at the level of system design rather than of an individual run. \citet{jia2026masfire} instead perturb execution directly: MAS-FIRE injects 15 intra-agent and inter-agent fault types through prompt modification, response rewriting, and message-routing manipulation, and reports that stronger foundation models do not uniformly improve robustness and that coordination topology is an equally decisive factor. These studies establish what fails and under which conditions, but a taxonomy label assigned to a completed run does not determine which edit to that run would have recovered the task.

\emph{Failure diagnosis.}
A second line therefore localizes the failure within a specific trajectory. \citet{zhang2025whichagent} formulate automated failure attribution, that is, identifying the agent and the decisive step responsible for a failure, and release Who\&When, which contains failure logs from 127 LLM-based multi-agent systems with fine-grained annotations. Their evaluation quantifies how hard the task is: the best of the three methods they study reaches 53.5\% accuracy on the responsible agent but only 14.2\% on the decisive step, and state-of-the-art reasoning models do not reach practical usability. \citet{barke2026agentrx} target the same localization problem with explicit evidence rather than direct log inspection; AgentRx synthesizes constraints from domain policy and tool schemas, evaluates them step by step, and emits an auditable log of constraint violations from which a judge identifies the critical failure step and category, evaluated on a released benchmark of 115 annotated failed trajectories. \citet{li2026errorprobe} propose ErrorProbe, which combines symptom-driven backward tracing, tool-grounded hypothesis validation, and a verified episodic memory. All three operate on a terminated trajectory and return an explanation, so a correct localization narrows the search for a correction without supplying one, and none produces the inspection decision over a partial trajectory that runtime intervention requires.

\emph{Failure repair.}
A third line modifies the execution itself, with objectives that differ in what counts as success. \citet{xie2026sparktofire} model collaboration as a directed dependency graph, derive an early-stage criterion for error amplification, and add a genealogy-graph governance layer as a message-layer plugin that suppresses error propagation without altering the collaboration architecture; the objective is to prevent a local error from becoming false consensus rather than to recover a run that has already failed. \citet{liu2026safeagent} share this preventive orientation under a security objective: SafeAgent mediates actions around the agent loop through a stateful decision core and is evaluated on prompt-injection benchmarks, so it blocks unsafe actions rather than restoring task success. The two closest methods repair for the sake of task outcome. CausalFlow assigns each step a causal responsibility score by replacing it and re-executing the downstream computation, then generates minimally edited repairs whose validation yields contrastive pairs of a wrong and a corrected step~\citep{bonagiri2026causalflow}; it targets single-agent sequential traces, performs repair at deployment time without parameter updates, and leaves the use of its pairs for preference optimization to future work. DoVer operates on MAS with an explicit orchestrator and sub-agent topology, generates failure hypotheses from logs, and validates them through targeted interventions such as message or plan edits, scoring outcomes by resolution or measurable task progress rather than by attribution accuracy~\citep{ma2026dover}; it requires framework-level checkpoint and replay infrastructure, and reports that several distinct interventions can independently repair the same failed run. ARIA differs in what it produces. Rather than searching for a correction on each incident at test time, it aggregates execution-verified repairs across MAS frameworks, benchmarks, and foundation models into a single corpus and trains one structured predictor on it, so that a diagnosis, a cut point, and a replacement continuation are emitted in a single forward pass, a prefix-level inspection gate decides on a partial trajectory when to emit them, and no checkpoint or replay hook is required inside the host system. Table~\ref{tab:related-work-comparison} in Appendix~\ref{app:related-work-comparison} summarizes the resulting capability profile.

\section{Reproducibility and Implementation Details}
\label{app:reproduction}
\label{app:technical-details}

\subsection{Execution Configuration and Benchmark Evaluation}
\label{app:execution-configuration}

Each experiment is specified by a configuration file that fixes the MAS,
benchmark split and task selection, host model,
repair model, API protocol, decoding settings, execution limits, and evaluator.
We distinguish the source model that generated a failure from the endpoint
used for repair generation and paired reruns; these need not be identical.
Each sample records both identities and its generation settings.

The four native frameworks use separate dependency environments, while
benchmark loading and evaluation use the default environment.
The native revisions of OpenManus, MetaGPT, Magentic-One, and AG2 are
pinned and recorded with each experiment.
Magentic-One additionally requires a Playwright Chromium installation.
Repository-local adapters normalize native messages, tool activity, and
workspace artifacts without modifying these checkouts.

\begin{table}[htbp]
\centering
\small
\caption{Benchmark verification. Native termination alone does not establish
benchmark success.}
\label{tab:appendix-evaluators}
\begin{tabularx}{\linewidth}{@{}lX@{}}
\toprule
\textbf{Benchmark} & \textbf{Verification} \\
\midrule
GSM-Plus & Extracted numerical answer checked against the reference. \\
MMLU & Extracted option checked against the reference choice. \\
HumanEval & Generated code executed against the benchmark tests. \\
OlympiadBench & Answer-bearing mathematics tasks scored with the public mathematical-answer scorer. \\
GAIA & Validation tasks and attachments; public quasi-exact answer matching. \\
SWE-bench-Lite & Workspace patch evaluated by the official Docker harness. \\
\bottomrule
\end{tabularx}
\end{table}

The maintained ten-task integration matrix uses OlympiadBench rows
1215--1224, the first ten answer-bearing English, text-only mathematics
examples; the initial proof-only rows lack final-answer targets.
This diagnostic selection does not define the paper's full task pools.
GAIA attachments require an authorized Hugging Face account.
Missing official reports, container failures, and unresolved evaluator
outputs retain an unknown correctness label rather than becoming wrong answers.

\subsection{Source Admission and Incident Construction}
\label{app:dataset-details}
\label{app:dataset-benchmark-pool-size}

A source cell fixes the framework, benchmark, and host-model configuration.
Within each source batch, we deduplicate by sample identity and consult the
latest result. A cognitive repair incident requires both native completion
and an explicit incorrect evaluator verdict. Transport, runtime, tool-backend,
and evaluator failures are recorded separately and receive no cognitive label.

The incident contains the public task, constraints, and normalized observable
events, each with a stable identifier, role, kind, content, and tool fields.
Native traces are preferred; messages provide a fallback when traces contain
no observable events. Available workspace artifacts can supply additional
events. The default budget is 64 events; longer traces retain the first
four and last 60 events. Localization is relative to this visible incident,
and cannot certify an omitted event as the earliest error.
Post-hoc feedback discloses only that the public evaluator marked the source
incorrect. Reference answers, hidden tests, gold patches, and private
evaluation details are excluded from model inputs. References remain in the
audit record solely for host-side evaluation.

Successful source runs supply non-intervention supervision. For $F>0$
selected failures and configured rate $\rho>0$, the default success quota is
$\max(1,\lfloor\rho F\rfloor)$, capped by available successful runs; otherwise
it is zero. The loader takes eligible successes in source order, unless an
explicit count overrides the quota. The configured rate therefore defines
a failure-relative quota, not Bernoulli sampling of successful trajectories.
The maintained repair configurations set $\rho=0.25$.
These success records invoke no repair generation or paired reruns.

\subsection{Localization, Trajectory Cuts, and Repeated Repair Trials}
\label{app:repair-construction}

Each trial makes three successive structured predictions: failure
localization, repair attribution, and a patch. Localization identifies the
first bad event(s), an earlier retained frontier, supporting evidence, and
confidence. Attribution conditions on this localization and specifies one
primary RA code, optional secondary codes, responsible roles, a diagnosis,
a recovery counterfactual, and a patch operator. The final patch contains
only a diagnosis, the retained cut-event identifier, and replacement events.
Schema and grounding checks validate event references and their order;
they do not establish that a diagnosis is causally correct.
The patch must copy the accepted attribution's diagnosis exactly and cannot
reproduce the discarded suffix unchanged.

\paragraph{Cut semantics.}
For an observed trajectory $\tau=(e_1,\ldots,e_T)$, a cut at $e_k$ retains
that event and discards the entire suffix after it. The proposed continuation is
\begin{equation}
\label{eq:appendix-splice}
\tau'=(e_1,\ldots,e_k)\mathbin{\Vert}
       (\hat e_1,\ldots,\hat e_L).
\end{equation}
A null cut denotes an empty retained prefix ($k=0$).
The cut must precede the earliest localized bad event. Replacement identifiers
must not collide with retained identifiers, and the model cannot fabricate
tool-result events. Patches contain at most eight replacement events and
12{,}000 characters in their canonical JSON serialization.

\paragraph{Independent trials and validation retries.}
The configuration file sets the number $n$ of complete
localization--attribution--patch chains. Each chain starts from the same
incident rather than refining a previous trial. Bounded structured-output
retries correct a stage's invalid response and are logged within that trial;
they are not additional independent trials. The maintained construction
configurations currently use $n=1$; 18 allow two structured-output retries per
stage and six allow eight.
Increasing $n$ preserves every trial, including
invalid generations and failed repairs. Independent calls do not imply
distinct outputs under deterministic decoding.

\paragraph{Paired execution and acceptance.}
The host runs one shared unpatched control and one patched execution per
complete trial, in separate workspaces under the declared repair endpoint.
The control receives the original task. A patched run receives the task,
retained prefix, attribution, and replacement events, with instructions to
enact the patch and continue. Both use the benchmark's registered evaluator.
A positive repair requires the patched execution to pass and the fresh
control to fail; a passing source-independent
rerun is insufficient. All complete trials are executed, and the first
verified success is selected as the canonical repair. Every verified trial
can contribute training data; there is no model vote or ranking over trials.

The native backend reconstructs a continuation context rather than restoring
the framework's complete internal state. Paired success therefore validates recovery
under this reconstruction protocol, not checkpoint-exact causal replay.

\subsection{Training Projection and Data Partitioning}
\label{app:training-projection}

The detailed repair records are the source of truth. Aggregation derives
training examples from these records using a shared projection procedure,
rather than concatenating potentially stale per-experiment exports. It deduplicates
record identifiers, rejects conflicting records, checks answer blindness,
and requires each DPO input/chosen pair to match its structured-repair SFT row.
Counts of source tasks, repair trials, verified repairs, and training rows
are distinct: one incident may yield many gate prefixes and multiple
verified patches. Corpus and split manifests retain these counts separately.

\paragraph{Supervision eligibility.}
Valid model-generated localization and attribution labels from admitted
failures can train their respective tasks even when the resulting patch
fails. Structured-repair positives and repair preferences require a verified
passing patch paired with a failing control; heuristic fallbacks are excluded.
For the selected localization with first bad index $b$, source-prefix gate
targets are $z_t=\mathbb{1}[t\geq b]$. If no valid first-bad index is
available, only the terminal failure prefix receives a positive label.
Success-source prefixes receive zero. Prefixes extending into the selected
verified replacement also receive zero; this is a projected training label,
not a separate evaluator verdict for every intermediate prefix.
Gate inputs always omit evaluator feedback.

\paragraph{Preference negatives.}
The chosen output is a verified patch. A schema-valid, executed failed trial
is used as the rejected output when available. Otherwise, the projector
encodes the observed failed suffix after the chosen cut as a repair record,
retaining at most three events and reducing content to meet schema limits.
This fallback shares the chosen diagnosis and cut; it is not a newly executed
negative. Both outputs use the same event-identifier conventions, and
rejection labels remain audit metadata rather than model-visible cues.

\paragraph{Splits and held-out conditions.}
The splitter assigns task identifiers, not individual prefixes or trials,
to nominal 60\%/20\%/20\% train/validation/test partitions using a seeded
hash within source-cell strata. A task occurring in several cells receives
one global assignment, with its lexicographically first cell defining its
stratum. Consequently, per-cell proportions are approximate; strata with
fewer than three tasks are assigned to training. All derived rows inherit
their task's assignment. Training-view audits additionally check held-out
overlap by task identifiers and normalized task text.
For unseen-system evaluation, the held-out axis is
excluded from training, calibration, and checkpoint selection.
Subset caps and task identities are fixed by each study's manifest; rates
from different task pools are not directly comparable.

\paragraph{Experimental subsets.}
\label{app:rq-specific-subsets}
RQ1 uses the complete task-level partitions for its in-domain post-hoc and
runtime comparisons, whereas RQ2 and RQ3 use separate fixed subsets to make
repeated backbone-scaling and ablation runs computationally tractable. In
unseen-system conditions, all trajectories from the held-out framework are
excluded from training, calibration, validation, and checkpoint selection,
with RQ2 holding out Magentic-One or AG2 and RQ3 holding out Magentic-One;
the resulting checkpoints are evaluated on the corresponding fixed test
subset.

\subsection{Post-Training Procedure}
\label{app:post-training}

The three methodological stages comprise five sequential training phases:
structured SFT, repair DPO, diagnosis transfer, repair GRPO, and gate GRPO.
The reference trainer uses AdamW and LoRA, with scaling $2r$ for rank $r$
and zero LoRA dropout. Exporting merged weights does not change this into
full-parameter fine-tuning. Prompts contain task instructions and canonical
JSON. Input reduction preserves event identities and JSON structure; targets
are never truncated into incomplete JSON. Oversized or ungrounded records
are handled by the configured admission policy and recorded in the report.

\subsubsection{Stage 1: Structured Supervision}
\label{app:pt-stage1}

SFT predicts four target types: gate probability, localization, attribution,
and structured repair. Prompt tokens are masked from the causal language
model loss. In cascade training, attribution receives the recorded
localization and repair receives the recorded attribution; deployment uses
the model's own predictions, so errors can propagate between stages.
Feedback dropout removes evaluator feedback with the configured probability;
it does not paraphrase hidden evaluator information.

\subsubsection{Stage 2: Repair Preferences and Diagnosis Transfer}
\label{app:pt-stage2}

Repair DPO compares the chosen and rejected records under a shared input,
using frozen stage-start reference log probabilities. Diagnosis transfer
then constructs training-only incident families from sparse token-count
representations of feedback and repair edits. Similarity is the mean of
their cosine similarities; families are connected components of the
thresholded mutual-nearest-neighbor graph. Mining requires at least three
family members.

Candidates include the source diagnosis, policy-generated diagnoses, and
eligible sibling diagnoses. With model parameters fixed during mining,
each candidate is scored by the log-likelihood uplift it gives a sibling's
verified repair relative to the same input without that diagnosis.
The selected diagnosis maximizes calibration-sibling uplift subject to the
configured threshold and positive uplift on one reserved training sibling.
The lowest-scoring distinct candidate supplies the preference rejection.
The reserved sibling is part of training, not the validation or test set.
Diagnosis DPO is followed by SFT on the complete verified repair record with
the selected diagnosis substituted. This signal measures likelihood-based
transfer, not newly executed recovery.

\subsubsection{Stage 3: Execution-Grounded Updates}
\label{app:pt-stage3}

Repair optimization samples diagnosis branches and multiple patches per
branch. In cascade mode, each branch first generates grounded localization
and attribution. Valid patches are scored by the configured environment;
native environments execute the MAS, whereas evaluator-only and recorded
outcome backends do not establish native execution. The default repair reward is
\begin{equation}
\label{eq:appendix-repair-reward}
r=\mathbb{1}[\mathrm{pass}]
  -\mathbb{1}[\mathrm{harmful}]-0.05\,b,
\end{equation}
where $b$ is the backend-recorded execution cost. Static rejection receives
$-1$ without execution; infrastructure-invalid or unresolved outcomes
provide no correctness reward.
Diagnosis value combines source and available cross-cell sibling rewards
with default transfer weight $0.5$. Without a scoreable sibling it falls
back to the source reward, which is recorded as missing transfer evidence.
Diagnosis/localization outputs receive advantages standardized across
branches; patch outputs receive advantages standardized within a branch.
Groups without sufficient reward variation are discarded.

The trainer uses a reference-policy penalty (default coefficient $0.02$)
and a token-level GRPO update. Its current single inner update computes the
old log probability by detaching the same forward pass; the probability
ratio is therefore one at evaluation. Although clipping bounds are
configured, this does not constitute multiple clipped PPO epochs.

Gate optimization evaluates sampled decisions on historical trajectory
prefixes, with the first trigger invoking repair. It uses a paired
same-host control and the reward
$S_{\mathrm{int}}-S_{\mathrm{base}}-0.05\,\Delta\mathrm{EB}$.
Historical-prefix rollouts are distinct from live interception starting at
task initialization. The scalar gate mode updates a global logit bias while
freezing repair parameters; the contextual mode updates shared LoRA
parameters with gate regularization and a repair SFT anchor, and thus does
not freeze the repair model. The selected mode is part of checkpoint identity.

\paragraph{Concrete reference recipe.}
The reference training configuration specifies Qwen2.5-7B-Instruct,
seed 42, one epoch, learning rate $2\times10^{-4}$,
LoRA rank 16, gradient accumulation 8, BF16, and gradient checkpointing.
It sets a 4{,}096-token total budget, a 1{,}024-token prompt budget,
feedback dropout $0.25$, and DPO $\beta=0.1$.
Its repair GRPO budget is 12 incidents, three diagnosis branches and three
patches per branch, one transfer sibling, and four groups per update;
gate optimization uses three rollouts, at most four probes, and one outer
round. These are reference-recipe settings, not universal settings for all
released checkpoints: experiment-specific configuration files and training reports determine
the actual row admissions, sampling budgets, and optimizer updates.

\subsection{Integration with Native MAS Runtimes}
\label{app:native-runtime}

ARIA is attached at the native model-response boundary. Before a generated
response is delivered to its framework, the adapter normalizes the actual
model request and pending response into an observable prefix.
No final evaluator verdict or future event is available. The inspection
gate predicts a probability; a value below the configured threshold
delivers the original response. A trigger invokes localization, attribution,
and structured repair, followed by host-side validation and capability checks.
The host, rather than the repair model, decides whether the proposal can be
enacted.

The current four native integrations support at most one pending plain-text
message substitution per run. An applicable patch cuts immediately before
the pending response and supplies a single compatible replacement with the
same role and nonempty changed content. Tool calls, protected structured
responses, multi-event continuations, and historical rollback cannot be
enacted by this interface; such proposals remain shadow records.
In shadow mode all proposals are recorded without substitution.
The framework continues normally after an applied replacement.

Malformed predictions are recorded as method errors. Under fail-open
operation, the original response is delivered, but the error is not relabeled
as a valid negative gate decision. Audit records retain the input prefix,
raw ARIA calls, proposed patch, original and delivered responses, and
enactment status. A proposal or unchanged replacement alone is not evidence
of an intervention.

\paragraph{Baseline adaptations and controls.}
The comparison methods are repository-local adaptations over the common
event and execution interfaces, rather than executions of official baseline
repositories. Post-hoc methods receive observable failed logs, public
feedback, and task constraints. Zero-shot emits the common repair schema;
Who\&When, AgentRx, and ErrorProbe add their localization procedures before
shared repair generation. DoVer generates interventions for segmented trials
and evaluates them through the declared replay backend.
Live Spark to Fire and SafeAgent adaptations receive prefixes without
feedback and share the pending-text substitution limit; SafeAgent is a
text-boundary reference, not its full safety architecture.
Within a comparison panel, task identities, host configuration, execution
limits, and decoding settings must match. Method-specific prompts and
additional calls remain visible in the audit and cost records.

\subsection{Success, Harm, and Execution-Cost Accounting}
\label{app:metric-accounting}

\paragraph{Denominators.}
Let $\mathcal{P}$ be the fixed evaluation pool. Post-hoc evaluation contains
evaluator-confirmed failed source trajectories; runtime evaluation contains
all selected tasks, including those that never trigger intervention.
For a fully resolved pool of size $N$, let $s_i\in\{0,1\}$ denote benchmark
success, $q_i$ the graded score, and $q_i^0$ the paired untreated score.
The runtime composite harm indicator is
\begin{equation}
\label{eq:appendix-hir}
h_i=\mathbb{1}\!\left[
u_i\lor v_i\lor(a_i\land q_i<q_i^0)\right],\qquad
\mathrm{SR}=\frac{\sum_i s_i}{N},\quad
\mathrm{HIR}=\frac{\sum_i h_i}{N},
\end{equation}
where $u_i$ denotes a malformed generated prediction, $v_i$ an explicitly
recorded task-constraint violation, and $a_i$ an enacted intervention.
With binary grading, outcome harm reduces to a passing control becoming a
failure after intervention. Several harmful events on one task count once.
HIR is task-normalized, not normalized by gate triggers or by successful
controls. The pooled rate weights tasks equally, not framework or benchmark
cells equally. For repeated training seeds, the figure aggregator averages
the per-seed rates with equal seed weight.

\paragraph{Resolution and live-report scope.}
The strict evaluation reporter withholds headline SR/HIR until every planned
case is resolved; resolved-subset rates are progress diagnostics.
The live artifact reporter separately exposes terminal-pool host success
and a \emph{harmful-substitution lower bound}: score regressions with an
effective delivered substitution, divided by the full planned pool.
It reports unresolved cases and decision errors separately. This lower
bound excludes malformed-but-unapplied proposals and must not be substituted
for the composite HIR in Equation~\ref{eq:appendix-hir}.
Regressions without applied intervention are reported separately from
harmful substitutions.

\paragraph{Execution budget and repeated sampling.}
Execution budget is the mean sum of recorded full-MAS execution counts per
task. Runtime includes the initial execution; post-hoc budget excludes the
already acquired failed source, with source-inclusive cost reported separately.
Additional method-required reruns count toward cost. ARIA model calls, tokens,
and latency are reported separately: a single-run runtime budget does not
imply zero inspection overhead.
Dataset construction executes all configured trials. In the separate
post-hoc budget sweep, the reporter consumes candidates in order and stops
at the first evaluator-confirmed success. A runtime sweep cannot select
with hidden evaluator verdicts: the implemented multi-candidate reporter
uses agreement on public outputs, breaking ties by candidate order.
These evaluation selection rules are distinct from construction-trial
selection and from diagnosis/patch sampling during GRPO.

\subsection{Artifacts and Resume Semantics}
\label{app:artifacts}

Reproduction requires the complete configuration, repository and native revisions,
model/tokenizer identities, source and split manifests, checkpoint lineage,
and generation settings. Dataset artifacts retain all prompts, raw and
parsed structured responses, validation errors, paired outcomes, and
per-sample endpoint provenance. Training exports contain canonical inputs
and targets without the generation-prompt wrappers. Reports retain corpus
checksums, projection drift, admitted/rejected row counts, backend identity,
and completed updates; a configured budget is not evidence of completion.

Repair resume skips a sample only when its identifier is present in the
formal dataset. Otherwise, partial sample artifacts are removed and all
configured trials restart from localization. Each dataset identity has one
writer. Native batch resume separately skips completed samples and retries
failed samples in clean workspaces. These sample-level guarantees do not
imply exact optimizer/RNG restoration from intermediate training checkpoints.

\section{Repair Attribution Taxonomy}
\label{app:repair-attributions}

\begin{table*}[htbp]
\centering
\caption{Repair attributions comprise six groups and 26 categories.
Each tag names a category associated with a candidate patch operator;
colors indicate group membership.
Group and category names are shortened for presentation.}
\label{tab:repair-attribution}
\small
\setlength{\tabcolsep}{6pt}
\renewcommand{\arraystretch}{1.12}
\begin{tabularx}{\textwidth}{
  @{}
  >{\raggedright\arraybackslash}p{0.31\textwidth}
  >{\raggedright\arraybackslash
    \setlength{\baselineskip}{15pt}}X
  @{}
}
\toprule
\textbf{Group} & \textbf{Categories} \\
\midrule

\RAGroup{RAOne}{RA-1}{Understanding \& Planning}
& \RACategory{RAOne}{RA-1.1}{Task misunderstanding}\hspace{0.4em}%
  \RACategory{RAOne}{RA-1.2}{Contract misunderstanding}\newline
  \RACategory{RAOne}{RA-1.3}{Faulty decomposition}\hspace{0.4em}%
  \RACategory{RAOne}{RA-1.4}{Faulty role planning} \\
\addlinespace[5pt]

\RAGroup{RATwo}{RA-2}{Evidence Grounding}
& \RACategory{RATwo}{RA-2.1}{Insufficient retrieval}\hspace{0.4em}%
  \RACategory{RATwo}{RA-2.2}{Misgrounded evidence}\newline
  \RACategory{RATwo}{RA-2.3}{Extraction/aggregation error}\hspace{0.4em}%
  \RACategory{RATwo}{RA-2.4}{Unsupported claim} \\
\addlinespace[5pt]

\RAGroup{RAThree}{RA-3}{Reasoning \& Decisions}
& \RACategory{RAThree}{RA-3.1}{Invalid inference}\hspace{0.4em}%
  \RACategory{RAThree}{RA-3.2}{Computation error}\newline
  \RACategory{RAThree}{RA-3.3}{Finite evidence as proof}\hspace{0.4em}%
  \RACategory{RAThree}{RA-3.4}{Incorrect decision} \\
\addlinespace[5pt]

\RAGroup{RAFour}{RA-4}{Agent Coordination}
& \RACategory{RAFour}{RA-4.1}{Lost shared state}\hspace{0.4em}%
  \RACategory{RAFour}{RA-4.2}{Ignored correct evidence}\newline
  \RACategory{RAFour}{RA-4.3}{Failed handoff}\hspace{0.4em}%
  \RACategory{RAFour}{RA-4.4}{Non-progress cycle} \\
\addlinespace[5pt]

\RAGroup{RAFive}{RA-5}{Tools \& Implementation}
& \RACategory{RAFive}{RA-5.1}{Unexecuted tool call}\hspace{0.4em}%
  \RACategory{RAFive}{RA-5.2}{Incorrect tool call}\newline
  \RACategory{RAFive}{RA-5.3}{Unrecovered tool failure}\hspace{0.4em}%
  \RACategory{RAFive}{RA-5.4}{Unmaterialized artifact}\newline
  \RACategory{RAFive}{RA-5.5}{Faulty implementation} \\
\addlinespace[5pt]

\RAGroup{RASix}{RA-6}{Verification \& Completion}
& \RACategory{RASix}{RA-6.1}{Missing verification}\hspace{0.4em}%
  \RACategory{RASix}{RA-6.2}{Inadequate verification}\newline
  \RACategory{RASix}{RA-6.3}{Accepted failed check}\hspace{0.4em}%
  \RACategory{RASix}{RA-6.4}{Premature termination}\newline
  \RACategory{RASix}{RA-6.5}{Output handling error} \\

\bottomrule
\end{tabularx}
\end{table*}

Each attribution identifies one primary localized mechanism; secondary
codes can record contributors. For example, treating finite search as proof
(RA-3.3) suggests replacing the empirical claim with a deductive argument,
whereas an ignored failed check (RA-6.3) calls for revising the decision to
accept the result. The attribution proposes what to change; paired
execution determines whether the resulting patch provides positive
repair supervision.

\section{Repair Utility: Proof and Counting Unit}
\label{app:repair-utility-proof}

\begin{proof}[Proof of Proposition~\ref{prop:repair-utility-score}]
For valid source execution $i$, let $F_i$ indicate source failure and $R_i$
indicate successful repair under the fixed repair protocol. Then
$\mathbb{E}[F_iR_i]=(1-a(c))r(c)$. Summing over $N$ source executions
gives $\mathbb{E}[\sum_i F_iR_i]=N(1-a(c))r(c)=NS(c)$.
\end{proof}

The counting unit must remain fixed when estimating $r(c)$.
For a protocol with $n$ trials and one shared control, define repair success
at the incident level as at least one verified passing patch paired with
a failing control.
The utility then measures incidents yielding usable repairs per valid source
execution. It does not count the number of successful trial-level training
rows. Trial count, endpoint, and execution limits must be held fixed when
comparing source cells; unresolved outcomes are reported separately and
never counted as verified repairs.

\clearpage
\section{Related-Work Capability Comparison}
\label{app:related-work-comparison}

Table~\ref{tab:related-work-comparison} supplements Section~\ref{sec:related-work} with a comparison of the capabilities explicitly reported by each method.

\begin{table}[htbp]
\centering
\caption{Capabilities that each method explicitly reports, where \emph{verified repair} requires re-executing a corrective modification under a task evaluator and a check mark records a reported capability rather than a performance comparison.}
\label{tab:related-work-comparison}
\footnotesize
\setlength{\tabcolsep}{3.5pt}
\begin{tabular}{@{}lcccccc@{}}
\toprule
\textbf{Work} & \makecell{\textbf{Failure}\\\textbf{taxonomy}} & \makecell{\textbf{Trace}\\\textbf{corpus}} & \makecell{\textbf{Step}\\\textbf{attribution}} & \makecell{\textbf{Test-time}\\\textbf{intervention}} & \makecell{\textbf{Verified}\\\textbf{repair}} & \makecell{\textbf{Learned}\\\textbf{predictor}} \\
\midrule
Multi-Agent Risks~\citeyearpar{hammond2025multiagentrisks} & \checkmark & & & & & \\
MAST~\citeyearpar{cemri2025mast} & \checkmark & \checkmark & & & & \\
MAS-FIRE~\citeyearpar{jia2026masfire} & \checkmark & \checkmark & & \checkmark & & \\
Who\&When~\citeyearpar{zhang2025whichagent} & & \checkmark & \checkmark & & & \\
AgentRx~\citeyearpar{barke2026agentrx} & \checkmark & \checkmark & \checkmark & & & \\
ErrorProbe~\citeyearpar{li2026errorprobe} & & & \checkmark & & & \\
Spark to Fire~\citeyearpar{xie2026sparktofire} & & & & \checkmark & & \\
SafeAgent~\citeyearpar{liu2026safeagent} & & & & \checkmark & & \\
CausalFlow~\citeyearpar{bonagiri2026causalflow} & & & \checkmark & \checkmark & \checkmark & \\
DoVer~\citeyearpar{ma2026dover} & & & \checkmark & \checkmark & \checkmark & \\
\midrule
\textbf{ARIA (proposed)} & \checkmark & \checkmark & \checkmark & \checkmark & \checkmark & \checkmark \\
\bottomrule
\end{tabular}
\end{table}

\end{document}